\documentclass[pdflatex,sn-mathphys-num]{sn-jnl}% Math and Physical Sciences Numbered Reference Style
\usepackage{graphicx}%
\usepackage{multirow}%
\usepackage{amsmath,amssymb,amsfonts}%
\usepackage{amsthm}%
\usepackage{mathrsfs}%
\usepackage[title]{appendix}%
\usepackage{xcolor}%
\usepackage{textcomp}%
\usepackage{manyfoot}%
\usepackage{booktabs}%
\usepackage{algorithm}%
\usepackage{algorithmicx}%
\usepackage{algpseudocode}%
\usepackage{listings}%
\usepackage{epstopdf}
\usepackage{float}%
\usepackage{caption}
\usepackage[T1]{fontenc}
\usepackage{lmodern}
\usepackage{microtype}
\theoremstyle{thmstyleone}%
\newtheorem{theorem}{Theorem}%  meant for continuous numbers
\newtheorem{proposition}[theorem]{Proposition}% 

\theoremstyle{thmstyletwo}%
\newtheorem{remark}{Remark}%

\theoremstyle{thmstylethree}%

\begin{document}

\title[Article Title]{Hybrid Parallel Bessel - KAN Model: a Kolmogorov-Arnold network-integrated framework for developing a novel architecture to reconstruct the fused-basis approach via a parameter free quadratic interaction}


\author[1]{\fnm{Mintu} \sur{Pathak}}\email{onelessmintu1994@gmail.com}

\author[2]{\fnm{Meshv} \sur{Patel}}\email{pmeshw12012004@gmail.com}

\author*[2]{\fnm{Mohendra} \sur{Roy}}\email{mohendra.roy@ieee.org}

\author*[1]{\fnm{Jwngsar} \sur{Brahma}}\email{jwngsar@gmail.com}

\affil[1]{\orgdiv{Department of Applied Sciences}, \orgname{Tezpur University}, \orgaddress{\street{Napaam}, \city{Tezpur}, \postcode{784028}, \state{Assam}, \country{India}}}

\affil[2]{\orgdiv{Department of ICT}, \orgname{Pandit Deendayal Energy University}, \orgaddress{\street{Raysan}, \city{Gandhinagar}, \postcode{382007}, \state{Gujarat}, \country{India}}}


\abstract{The recent emergence of the Kolmogorov-Arnold Network has challenged the traditional Multi-Layer Perceptron (MLP) framework by placing learnable activation functions on network edges instead of learnable weights, which offers promising theoretical scaling and interpretability. Yet, scaling KANs to complex spatial data, such as image recognition, exposes limitations in their default B-Spline basis functions, resulting in computational bottlenecks and the curse of dimensionality. While recent hybrid architectures attempt to resolve representational blind spots by mixing multiple basis families on a single edge, this monolithic approach introduces severe structural limitations. Specifically, superimposing localized B-Splines with structural, wave-like Bessel functions creates catastrophic geometric overlap, which drives the design matrix into severe numerical ill-conditioning ($\kappa(G) \approx 10^{16}$) and parameter explosion. In an attempt to resolve this mathematical bottleneck, we have proposed the Parallel BesselKAN, a novel architecture that fundamentally restructures the fused-basis approach. Instead of forcing multiple function families to share the same edge, our model isolates a B-Spline branch and a low-order Bessel branch into distinct, parallel processing pipelines integrated via a parameter-free quadratic interaction. Through rigorous empirical benchmarking and spectral matrix analysis, we demonstrate that this parallelization acts as a highly efficient internal ensemble. The decoupled architecture successfully bypasses the computational and stability bottlenecks of monolithic additive models, compressing the parameter footprint by 35\% while achieving a competitive peak accuracy of 97.99\% on the MNIST benchmark. Although standard parameter-matched MLPs currently retain a hardware-level computational speed advantage for pure computer vision tasks, Parallel BesselKAN provides a mathematically robust, parameter-efficient framework that stabilizes mixed-basis networks, paving the way for scalable spatial adaptations such as Convolutional KANs.}




\maketitle

\section{Introduction}\label{sec1}

One of the fundamental blocks of a neural network, more specifically the multilayer percepton, is a special mathematical object called the activation functions. In the traditional neural network composed of multiple layers and each layer itself composed of nodes where nodes in two consecutive layers are interconnected by edges, the activation functions sit on those nodes that provides non-linearity to the features being transmitted throughout the neural network architecture. Usually, these activation functions are fixed, and often comes with parameters that maybe trainable depending on the functionality of the neural network architecture. Researchers and scholars have explored hundreds of activation functions for various kinds of neural network. The most popular being the ReLU, Sigmoid, Tanh and rest being a variant of other well-known activations, each trying to compete with one another in regards to computational cost \& speed as well as accuracy on the dataset that the models equipped with these activation functions were tested on.

However, over the last few years, there has been a paradigm shift in how we view about these activation functions, with the recent emergence of Kolmogorov Arnold Network \cite{liukan2025}, which was since extended into a broad family of corresponding architectures \citep{somvanshi2024kan}, a novel neural network architecture where instead of placing the activation functions on the nodes of a neural networks, it is placed on the edges connected to a node where the addition of the output of those functions take place. This architectural innovation shows promising potential for model interpretability and accuracy. First proposed in 2023 in their paper, Liu et al \citep{liukan2025} stated that the KAN architecture takes the advantage of the Kolmogorov Arnold Representation Theorem, which lends KAN a leverage in faster neural scaling laws, and local support. Most significant of them all, KAN provides researchers to inject inductive biases and collaborate with the model to discover laws, which does away with the black box problem that traditional MLPs come with. However, deployment of these KAN models is not without its challenges: because of the inherent tension between the expressive capacity of the chosen basis functions and the computational cost required to train them. So, optimizing these basis representations becomes crucial as model’s complexity grows. Finding out the right function basis that maximizes feature extraction without adding prohibitive computational overhead is essential for scaling KANs to broader, more complex real-world applications.

In this paper \citep{liukan2025}, the Liu et al. used the B-Spline family of functions as their default basis to parametrize the functions on the edges. While B-Splines have shown to excel at capturing localized, piecewise continuous patterns, Liu et al. pointed out some limitations of using B-Splines, that being its disadvantages relating to its curse of dimensionality (COD), that adds computational cost by introducing grids, being only defined on a bounded region and most importantly, contributing to the slow training speed in KANs, owing to the fact that different spline activation functions cannot leverage batch computation as efficiently as MLPs with shared fixed activations. Also, in the research field of image task recognition, like in the study done by Yu et al. (preprint) \citep{yu2024kanmlp}, they have made a raw comparison between KAN and MLP based neural network and have found that KAN models performed poorly in computer vision tasks than MLP models. The authors of that paper stated that KAN’s advantage lies within their B-Splines basis functions that made it perform better on symbolic formula representation than MLPs, which was the only kind of task that the authors have seen KAN outperformed MLP through their ablation studies. Even then, they demonstrated that MLP with B-Spline as its activation also performed just as well as KAN based models, if not better.

In light of these limitations, researchers found that the default KAN is insufficient when it comes to their performance on image-based task. There are certain studies that demonstrated that when KAN is hybridized with a different architecture like in the study done by Ullah et al. \citep{ullahkanvit} in which the authors mention about developing a KAN-ViT hybrid, by replacing traditional MLPs in the transformers \citep{vaswani2017} of Vision transformers (ViT) \citep{dosovitskiy2021} with KAN. This were tested against the traditional MLP-ViT and what they’ve found is that KAN-ViT didn’t just performed better in image based task compared to MLP-ViT, they also mitigated the problem of “catastrophic forgetting” \citep{mccloskey1989, french1999, kirkpatrick2017, parisi2019}. However, there are also researches done to overcome the limitations of B-Splines in KAN that showed promising result in image-based task, by exploring alternative mathematical families of basis functions that replaces the default B-Splines. For example, Radial Basis Functions were introduced in RBF-KAN (Chao, Liu et al. \citep{rbfkan2026}), and Chebyshev Polynomials were utilized in Cheby-KAN (Mahmoud, Pester et al. \citep{chebykan2025}). Pushing this evolution further, recent state-of-the-art architectures have demonstrated that a hybrid approach i.e. mixing two or more families of basis functions yields superior results by combining their unique mathematical inductive biases. Notable examples of this include BSRBF-KAN (Thang Ta \citep{bsrbfkan2025}) and the Fully-Connected KAN (FC-KAN) (Thang Ta, Thai et al. \citep{fckan2026}). These hybrid models established that fusing different basis families effectively addresses the representational blind spots of using a single function type. Moreover, our proposed family of basis functions of our interest - the Bessel functions have also appeared in the KAN literature very recently, but only as a monolithic per-edge replacements for B-Spline functions. In the recent paper published by Yang et al. \citep{yang2026modifiedkan}, the authors made that replacement as an ablation model in a communications-classification study, while in another recent paper published by Poole et al. \citep{poole2026pkkan}, the replacement was made as  one candidate basis in a per-edge model-compression framework.

Our studies have shown that KAN does show promises for cutting edge results, simply because of the fact that KAN, unlike MLP, uses less parameter but with better accuracy, better interpretability, reduction in catastrophic forgetting, all of which is possible because of Kolmogorov Arnold Representation Theorem. We also believe, based on our studies on the KANs recent performance on image based task, that KAN might prove to be as useful as MLPs, but with less dependence on memory allocation. So, taking inspiration from the success of these hybrid and alternative architectures, our research introduces Bessel functions into the KAN framework in a hybrid fusion setting for the first time. While previous works have wholesale replaced B-Splines with Bessel functions \cite{yang2026modifiedkan, poole2026pkkan}, to the best of our knowledge no existing architecture have combined a Bessel branch with a spline branch through a parallel pipeline and a parameter-free quadratic interaction and exploiting their complementary inductive bias. As Bessel functions are naturally suited for capturing structural and radial wave patterns, our preliminary studies indicated that relying on high-order Bessel functions introduces severe computational bottlenecks and numerical instability. Hence, in order to resolve this, we propose the Parallel BesselKAN, a novel architecture that hybridizes traditional B-Splines with low-order Bessel functions via a parallel pipeline approach. By utilizing a parameter-free quadratic interaction fusion, our model functions as a highly efficient internal ensemble. This approach allows us to leverage the distinct mathematical strengths of both basis families simultaneously, achieving highly competitive accuracy within the KAN family while drastically reducing the network's parameter footprint and memory allocation compared to existing baseline models.

\section{Background Theory}\label{2}

In this section we discuss some requisite theories that is essential for discussing the mathematical
basis of the KAN architecture and then the family of functions that we used for our Hybrid model:

\subsection{ Kolmogorov-Arnold Network}\label{2.1}

This is a novel kind of neural network architecture introduced by Liu et al. \citep{liukan2025} in their paper published in
2024 that is seen as a paradigm shift from the traditional Multi Layer Perceptron (MLP) in that they
unlike in the traditional MLPs where we had fixed activation function on the nodes and learnable
weights on the edges of a neural network, KAN almost shifts it’s upside down, with the activation
being on the edge and learnable, while at the nodes, we have addition of all the activation functions
that is coming into the edges that are connected to it.

\begin{figure}[h]
\centering
\includegraphics[width=0.9\textwidth]{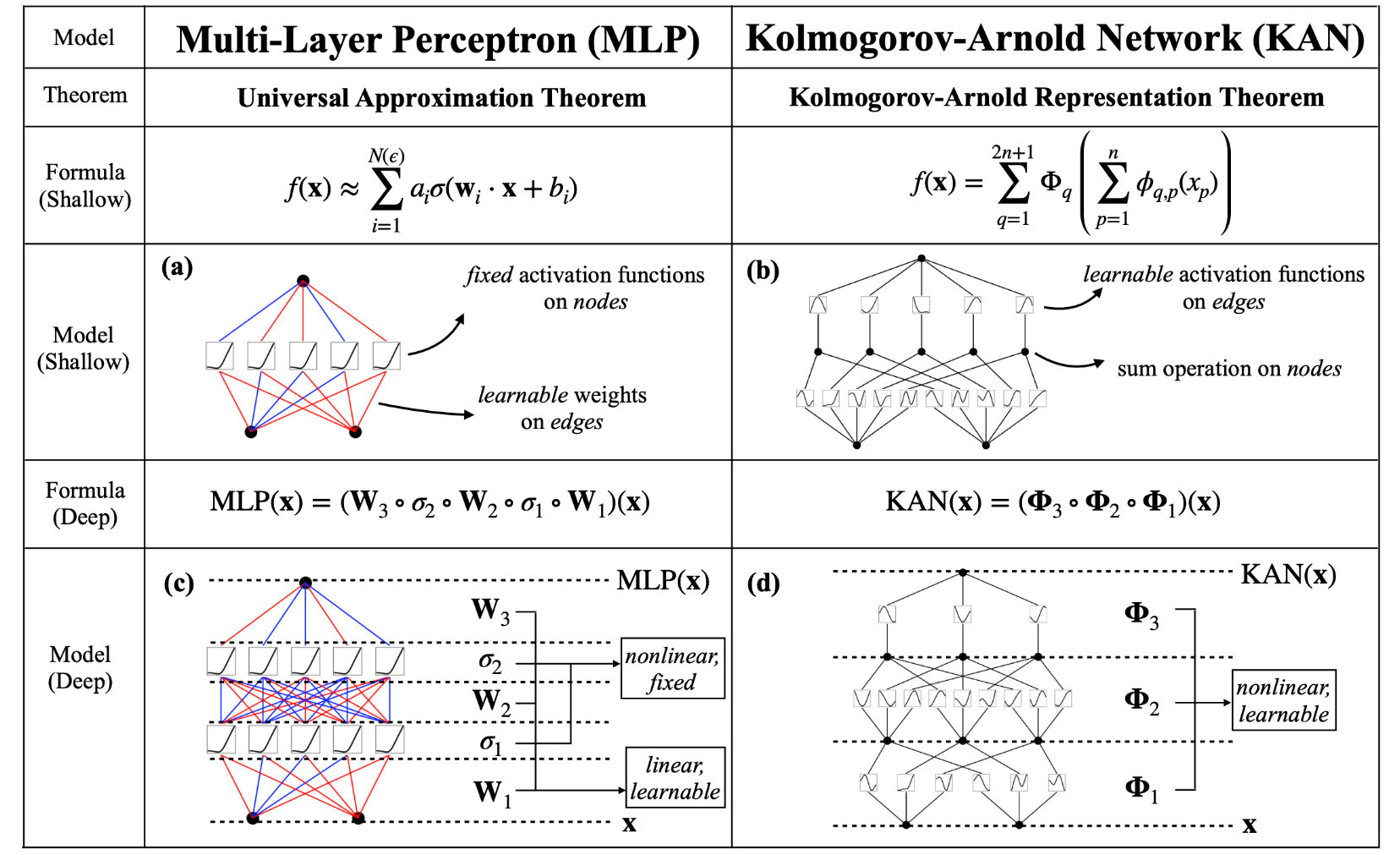}
\caption{Difference between MLP and KAN. Figure adapted from Liu et al. \citep{liukan2025} under a Creative Commons Attribution 4.0 International License.}\label{fig1.0}
\end{figure}

The main driving mathematical theorem behind KAN is the Kolmogorov Arnold Theorem \citep{kolmogorov1957, arnold1957}, which states
that states that any multivariate continuous function can be represented as a finite composition of
univariate functions and addition. Liu et al argue that if a function admits a smooth KA
representation, KANs can decompose high-dimensional learning into learning 1D functions, thereby
beating the curse of dimensionality (Liu et al. \citep{liukan2025}). This is in contrast to the Universal
Approximation Theorem \citep{cybenko1989, hornik1989, leshno1993} that serves as an inspiration for the MLPs, which guarantees that MLPs can
approximate any function, but provides no bound on how the required network size scales with
error \citep{barron1993}. In practice, MLPs suffer from the curse of dimensionality (COD), i.e. the required parameters can
grow exponentially with input dimension.

The authors went for KAN because they believe the Kolmogorov-Arnold theorem provides a more
structurally efficient path to representing real-world scientific functions than the generic Universal
Approximation Theorem that is mathematical basis behind the traditional MLPs. By replacing fixed
node activations with learnable edge splines, they created a network that is simultaneously more
parameter-efficient (better scaling laws), more accurate (especially for structured scientific data),
and fundamentally interpretable (enabling human-AI scientific discovery).

\subsection{Kolomogorov-Arnold Representation Theorem}\label{2.kart}

The theorem states that: every scalar valued, multi-variable function defined on a closed bounded subspace of a real Euclidean space can be expressed as as a finite composition of univariate functions and addition. That is, if $f:\left[0,1\right]^n \to \mathbb{R}$ is any multivariate, scalar valued function, then, for any $\mathbf{x} = (x_1,x_2,\ldots,x_n) \in \left[0,1\right]^n$, $f$ can be representated as-
\begin{equation}
f(\mathbf{x}) = f(x_1,x_2,\ldots,x_n) = \displaystyle\sum_{q=0}^{2n}\Phi_q\left(\sum_{p=1}^n \phi_{q,p}(x_p)\right)
\end{equation} 
where $\phi_{q,p} : \left[0,1\right] \to \mathbb{R}$ and $\Phi_q : \mathbb{R} \to \mathbb{R}$. The authors argue that if a function admits a smooth KA representation, KANs can decompose high-
dimensional learning into learning 1D functions, thereby beating the curse of dimensionality. Classical extensions by Sprecher \citep{sprecher1965} and Lorentz \citep{lorentz1966} further refined this representational capability.
In the original paper, the authors chose B-Spline functions as the family of basis that they used to
parametrized the learnable univariate functions on each edges.

\subsection{B-Spline Functions}\label{2.B}

B-splines (Basis splines) are piecewise polynomial functions that are used to construct smooth
curves and surfaces \citep{deboor1978}. These curves called the B-Spline curves are made by joining polynomial
segments of degree $k$ across a grid of points called knots. In the original paper, they typically used
polynomial degree of 3. Some of the key features of such curves is that they have local support, in
the sense that each B-Spline basis function that is used to generate the curves is non-zero only on very limited interval (between adjacent knots). This means changing one coefficient only affects a local
region of the curve, not the entire function. The curve also maintains smooth features, because at
the knots where polynomial pieces meet, the function is guaranteed to be continuous and have
continuous derivatives up to order $k-1$. For cubic splines ($k=3$), which was used in the original
paper, this means the curve is twice continuously differentiable ($\text{C}^2$)

What happens inside KANs is that each activation function on the edge is represented as a linear
combination of B-spline basis functions $B_i(x)$ with learnable coefficients $c_i$:
\begin{equation}
\text{spline}(x) = \sum_i c_i B_i(x)
\end{equation}
where the network learns the coefficients $c_i$ through gradient descent.

In summary, B-splines are used in KAN because they provide a learnable, smooth, locally adjustable,
and resolution-adaptive way to represent the univariate functions required by the Kolmogorov-
Arnold theorem. They give KANs their signature capabilities: high accuracy via grid extension,
resistance to catastrophic forgetting via locality, and the ability to decompose high-dimensional
problems without falling prey to the curse of dimensionality.

\subsection{Bessel Functions}\label{2.Bes}

Bessel functions are a special kind of functions that are solutions to a particular type of
ordinary differential equation, called the Bessel’s equation \citep{watson1922}:

\begin{equation}
x^2\frac{d^2y}{dx^2}+x\frac{dy}{dx}+\displaystyle\left(x^2-\alpha^2\right)y=0
\end{equation}

Where $\alpha$  is the order of the equation which is usually a complex number. These equations appear in
problems involving wave motion, heat conduction and other physical phenomenon with
circular/cylindrical symmetry. In general, there are two kinds of Bessel functions: The Bessel
functions of the $\text{1}^{st}$ kind given by
\begin{equation}
J_{\alpha}(x) = \displaystyle\sum_{m=0}^{\infty}\frac{(-1)^m}{m!\Gamma(m+\alpha+1)}\displaystyle\left(\frac{x}{2}\right)^{2m+\alpha}
\end{equation}
When $\alpha$ is either an integer or positive, these Bessel functions have finite values at the origin,
whereas if $\alpha$ is a negative non-integer, these functions are divergent at the origin. Also, for non-
integer $\alpha$, $J_{\alpha}(x)$ and $J_{-\alpha}(x)$ are two linearly independent solutions to the Bessel’s equation and so
they form a solution space for the Bessel’s equation. However if $\alpha$ is an integer, then $\alpha$, $J_{\alpha}(x)$ and $J_{-\alpha}(x)$ wouldn't be linearly independent because for any integer $n$, $J_{-n}(x) = (-1)^nJ_n(x)$. Then in
this case the other linearly independent solution is found through the Bessel functions of $\text{2}^{nd}$ kind. This type of Bessel functions are solutions of the Bessel differential equation that have a singularity
at the origin and are multivalued. Denoted by $Y_{\alpha}(x)$, where $\alpha$ is a non-integer, its formula is given by
\begin{equation}
Y_{\alpha}(x)=\displaystyle\frac{J_{\alpha}(x)\cos(\alpha\pi)-J_{-\alpha}(x)}{\sin(\alpha x)}
\end{equation}
When dealing with integer order $n$, the Bessel function of $\text{2}^{nd}$ kind is defined as $Y_n(x)=\displaystyle\lim_{a\to n}Y_{\alpha}(x)$

\section{Formulation of the Hybrid Model}

As stated before, there were many good reasons why the Liu et al \citep{liukan2025}. chose B-Splines as their
baseline family of basis functions. Because KAN leverages the KART to theoretically decompose a
multivariate function into its additive super-composition of univariate functions, B-Splines provide a
flexible, differentiable way to represent any smooth univariate function. In this case, each “weights”
in KAN is replaced by a univariate B-Spline curves. Because KANs reduce the high-dimensional
learning problem to learning multiple univariate functions, B-splines can approximate these
univariate components with high precision, which Liu et al. have found to be more
efficient than MLPs which try to approximate the same high-dimensional functions with commonly
known activations such as ReLU. Another unique advantage that the paper has noted about B-Splines is
their ability to “refine the grid” without requiring any retraining from scratch. The original paper
introduced a grid extension technique which starts off with a coarse grid (few intervals for each B-
Splines) and then refining those grids by fitting a new fine-grained spline to the old coarse one via
least squares. This has shown to improve accuracy as grids become finer, something MLPs cannot do
by simply adjusting a fixed activation function. Other advantages include the local support that B-
Splines have which lets KANs avoid catastrophic forgetting during continual learning and B-splines
being smooth and fully differentiable, which allows KANs to be trained end-to-end with
backpropagation.

However, B-Splines is not without its limitation, as the authors who used these basis family
have themselves reported. One of them being the curse of dimensionality (COD) of B-splines themselves.
Although with the help of B-Splines and with how the architecture of KAN is modelled, by combining
and leveraging the strength of B-Splines and MLPs, while avoiding their weaknesses, raw splines have
serious COD problem because of their inability to exploit their compositional structures, i.e. if one
tries to use B-splines to approximate a high-dimensional function directly in the full input space, the
required grid points grow exponentially with dimension. The authors have also noted is that KANs
appear "hopelessly expensive" at first glance because each MLP weight parameter becomes a KAN
spline function (Section 2.2 from Liu et al. \citep{liukan2025}). A KAN with depth $L$, width $N$ and $G$ grid intervals have $O(N^2LG)$ parameters, whereas an MLP only needs $O(N^2L)$ parameters. While the authors argue that KANs compensate by
requiring much smaller widths, the raw parameter count is inherently higher. B-Splines also have
issue with them being defined on a bound region, because during training, activation values can
evolve outside that region. The authors proposed a remedy by by updating grids on the fly according
to input activations, but this adds implementation complexity and is described as a non-trivial
optimization challenge. The biggest issue with B-Splines is that B-Spline causes a bottleneck to the
KANs making it typically slower than traditional MLPs. The bottleneck is due to the fact that different
spline activation functions cannot leverage batch computation (large data through the same
function) as efficiently as MLPs with shared fixed activations. They classify this as an engineering
problem rather than a fundamental limitation, but it remains the "biggest bottleneck" of the
architecture.

In the domain of image tasks, there are comparative studies being made with KANs against
other traditional MLPs, one of the significant ones includes the study done by Yu \& Wang \cite{yu2024kanmlp} where they
made a fair comparison between KAN based models with MLP based models on various tasks which
includes computer vision. They have found that under the same numbers of parameters or FLOPs.
KAN performed quite poorly compared to MLP on computer vision task (77.88\% vs 85.88\% test
accuracy). The computer vision task involved testing vanilla KAN and MLP (with GELU \citep{hendrycks2016}, RELU,
normalization) on 8 image datasets (such as MNIST \citep{lecun1998}, CIFAR-10/100 \citep{krizhevsky2009} variants etc.). They also found that
KAN had issue with forgetting more severe than that of MLP in a standard class-incremental learning,
which they claimed to have contradicted the claim made in the original KAN paper \citep{liukan2025}. Overall, the authors Yu \& Wang
claimed that while KANs have promises, it might not be a universal MLP replacement, because their
edge was niche and computer vision tasks favor MLPs over KAN, at the time the authors conducted
their experiments. However, this doesn’t mean that KAN is doomed to fail in computer vision task, as
we have noticed some glaring flaws in the author’s paper. The authors used vanilla KAN with basic
spline-based layers against an optimized MLPs. It didn’t test KAN hybrids or extensions designed for
computer vision which emerged after 2024. Secondly, the original KAN designs in the paper of Liu et
al. \citep{liukan2025} was made for small scale or science tasks, and they acknowledge KAN’s
scalability for computer vision as a future work. Furthermore, we have recent research showing KANs
excelling in computer vision task when adapted (like integrating KAN with Vision transformers \citep{ullahkanvit} or
adopting an alternate family of basis functions in KAN), often matching or beating MLPs in specific
scenarios like continual learning or efficiency.

In the paper published by Ullah \& Kim \citep{ullahkanvit} in 2025, they replaced the MLPs in ViT with KANs and what
they’ve found is that not only KAN-ViTs achieved comparable and even better accuracy on computer
vision tasks than MLP-ViT, but also significantly mitigated catastrophic forgetting, outperforming
MLP-ViT in retaining information and adapting to new task. While there is future scope where we
could do such replacement on other neural network models that similarly perform better on, our
main interest lies in the implementation of alternate families of basis function in KANs that helps the
KANs to perform better in certain task that the vanilla KANs normally fails at or at least performs as
better as MLPs. There are several studies that explores alternate KAN models with different basis
family of functions and also studies where they incorporated new family of basis function with the B-
Splines.

Some of these studies explores alternate KAN models such as:
\begin{enumerate}
\item \textbf{Cheby KAN \citep{chebykan2025}:} This replaces B-splines with Chebyshev polynomials to eliminate the costly spline-grid
updates, leveraging orthogonality and fast convergence for stabler training. It runs 14 times faster than B-
spline KAN-SchNet while achieving lower MAE on quantum-chemistry regression and stronger
generalization on MNIST/CIFAR-10.

\item \textbf{Fourier KAN \citep{fourierkan2026}:} This decomposes image features via FFT into amplitude (style) and phase (content),
then uses KAN to break these into learnable univariate basis functions for multi-level recombination.
This yields large gains in unknown-domain object-detection generalization (Diverse Weather \& Real-
to-Art benchmarks) without sacrificing real-time inference (110 FPS).

\item \textbf{Free-RBF KAN (preprint \citep{freerbfkan2026}):} This cuts B-spline overhead by adopting adaptive radial basis functions with
trainable centroids and smoothness, bypassing De Boor’s recursion entirely. It matches or exceeds B-
spline KAN accuracy on multiscale regression and physics-informed tasks (PINNs \citep{raissi2019}, DeepONet \citep{lu2021}) while
using fewer parameters and training faster.

\item \textbf{Rational KAN \citep{rkan2026}:} This introduces rational basis functions through Padé approximants and mapped
Jacobi polynomials to better capture asymptotic behavior and singularities that standard polynomials
miss. It outperforms alternatives on regression tasks with sharp peaks and delivers competitive
results on MNIST and physics-informed ODE/PDE solving.

\item \textbf{RBF KAN \citep{rbfkan2026}:} This model swaps B-splines for Gaussian radial basis functions to remove recursive grid
computation and improve interpretability via localized activations. It reduces parameters by 33–65\%
and accelerates training/inference by 65–99\% compared with B-spline KAN, while maintaining
comparable accuracy on function fitting and classification benchmarks.

\item \textbf{Wav-KAN (preprint \citep{wavkan2024}):} This model substitutes B-splines with wavelet transforms to enable multiresolution
analysis, capturing both high-frequency details and low-frequency trends while avoiding overfitting
to noise. It trains faster than Spl-KAN with fewer parameters and achieves higher test accuracy on
MNIST using wavelets such as DOG and Mexican hat.
\end{enumerate}
Each of these alternative models made replacement for B-Splines in KAN so that each
performs a specific task better than the baseline KAN, not necessarily these models are meant as a
total replacement of the original KAN models.

Taking inspiration from these, we explore the family of Bessel functions, which despite being well-studied in applied mathematics and appearing only as isolated spline replacements in prior KAN studies \cite{yang2026modifiedkan, poole2026pkkan}, has never been employed within a hybrid or fused-basis KAN architecture. As mentioned in the
previous section, these functions are the solutions to the $\text{2}^{nd}$ order differential equation called the
Bessel’s equation. These equations naturally appear in physical phenomenon that involves cylindrical
and circular symmetry. There are several properties why we think Bessel Functions make for a good
compliment for the B-Spline functions in KAN. The first property being that Bessel functions have a
naturally oscillating and decaying behavior, which makes it useful for capturing textures, edges and
radial patterns in images. The second properties being its radial symmetry – because these functions
are naturally defined in polar/cylindrical co-ordinates, it aligns well with circular or rotationally
features in images (an inductive bias previously successfully exploited in CNNs via Harmonic Networks \citep{worrall2017}). These functions also form a complete orthogonal basis in certain weighted and
can efficiently approximate functions with oscillatory components. Other important properties
include its rich theoretical background and being infinitely differentiable which is essential for
gradient based training in KAN. Furthermore, our motivation is empirically supported by recent evidence that spline-KANs approximate oscillatory special functions, including Bessel functions themselves, with near-perfect fidelity, substantially outperforming MLPs on such targets \cite{ramonpedrero2026kanmlp}, indicating that oscillatory structure is precisely the regime in which KAN edge functions excel.

However, there are some potential limitations of Bessel functions that would limit a full-on
replacement for B-Splines in KAN. The first being its lack of localization. Because Bessel functions of $1^{st}$ kind are globally supported and oscillatory with slow decaying amplitude, it becomes difficult to
represent sharp localized features such as edges, corner and fine textures. This initially suggest us
that we may need higher order Bessel functions to approximate a sharp discontinuity, leading to the
Gibbs phenomenon \citep{hewitt1979}. Another issue with Bessel functions is that they are highly non-linear and
oscillatory in nature. What this means is that Bessel functions have oscillating gradient that
results in many local minima and slower convergence, which would create several optimization
challenges. This would make it less efficient on non-radial and non-wavy image data.

Due to this, we’re inspired to take a hybrid approach with Bessel function where we mix
Bessel functions with the existing B-Splines in KAN. The development of our hybrid architecture
began by exploring a direct additive formulation, inspired by recent successes in mixed-basis models
such as BSRBF-KAN, and analogous in spirit to the pure Bessel-replacement KAN of Yang et al. \cite{yang2026modifiedkan}, but with both families present on every edge. In this initial design, the activation function on each individual network edge was
defined as the linear combination of a standard B-Spline basis and a Bessel function basis.
Mathematically, the activation on a given edge was formulated as
\begin{equation}
\phi(x) = \phi_{spline}(x)+\psi_{bessel}(x)
\end{equation}
The underlying hypothesis was that superimposing the localized, piecewise nature of B-Splines with
the structural, wave-like properties of high-order Bessel functions within the same functional space
would yield a richer feature representation.

While theoretically appealing, empirical evaluations of this direct additive approach exposed
significant structural limitations. Placing multiple basis families on every single edge caused an
immediate and severe parameter explosion. Furthermore, evaluating high-order Bessel functions
across the entire network resulted in a computational complexity that rendered comprehensive
ablation studies practically unfeasible. To mitigate this bottleneck, an intermediate architectural
correction was implemented: the Bessel function expansion was strictly truncated to Order-2,
utilizing only the low-order functions $J_0(x)$ and $J_1(x)$. Although this constraint successfully curtailed
the parameter bloat and restored computational tractability, it introduced a detrimental trade-off.
The reduced expressive capacity of the Order-2 restriction penalized the network's learning
dynamics, leading to a demonstrable decrease in benchmark accuracy compared to the initial
baseline experiments.

So, taking further inspiration from FC-KAN, we attempt to resolve the tension between expressive capacity and computational efficiency, by fundamentally restructuring the architecture from a per-edge additive model to a dual-pipeline
parallel framework. Instead of forcing both function families to share the same edge, the network is
split into two distinct, parallel processing pipelines: a dedicated B-Spline branch and an Order-2
Bessel branch. To ensure the overall parameter footprint remains highly efficient, the hidden
dimensions of these parallel pipelines are compressed (e.g., utilizing a 32 + 32 configuration). The
final feature representations from these parallel branches are fused using a parameter-free element-
wise interaction, defined as
\begin{equation}
O_{final} = O_{spline}+O_{bessel}+(O_{spline}\odot O_{bessel})
\end{equation}
To rigorously evaluate this quadratic interaction mechanism, we have also tested a simple additive fusion $O_{final} = O_{spline}+O_{bessel}$ and a concatenation fusion as comparative baselines, which would ensure a comprehensive ablation of the parallel pipeline's feature combination strategies. This Parallel
BesselKAN-Interact architecture effectively recovers the non-linear expressivity lost during the Order-
2 truncation without introducing additional learnable weights. By acting as a highly efficient internal
ensemble, this parallelization bypasses the computational bottlenecks of the initial additive approach
while securing superior predictive performance.

Now, in relation to the prior work on Bessel-based KANs, both Yang et al. \citep{yang2026modifiedkan} and Poole et al. \citep{poole2026pkkan} employed Bessel functions within KANs, but in a fundamentally different regime. Yang et al. evaluated a pure Bessel-replacement in KAN as a reference model for modulation classification, which has shown to under-perform in their distribution-aware design and it is consistent with our finding that a full replacement of B-Spline with Bessel functions sacrifices the spline branch’s local adaptivity. Poole et al. selected a single compact basis per edge for model compression, which is an objective approach that’s completely opposite to ours because none of their edges in their architecture could express cross-basis products, while our parameter-free quadratic fusion contributes additional dimensions of functional expressivity (See Theorem \ref{Thm4.2.1}) without learnable parameters. To our knowledge, Parallel BesselKAN is the first architecture to hybridize B-Spline basis with the basis of Bessel functions and the first to provide theoretical guarantees, as laid out in the next section exploring the mathematics behind our fused-basis KAN. Furthermore, where both previously published works operate on one-dimensional signal data, we target image classification, the setting in which KANs are known to be weakest \citep{yu2024kanmlp}

\section{Mathematical Structure}

In this section, we propose certain analysis, theorems and proposition that serves as a mathematical backbone of our model, giving a mathematical justification of how and why our model works.

\subsection{Conditioning and Numerical Characterization of the Interaction Space}

While subsequent propositions establish the theoretical linear independence of the mixed basis generated by piecewise polynomials and analytic Bessel functions, the empirical realization of this functional space over a compact domain $\Omega = [-1, 1]$ introduces structural ill-conditioning. To evaluate the stability boundaries of the Parallel BesselKAN, we systematically analyzed the design matrix $\Theta \in \mathbb{R}^{N \times M}$ and its associated Gram matrix $G = \Theta^\top \Theta$.

Across all evaluated basis configurations on a 500-point uniform collocation grid, the condition number of the explicit Gram matrix strictly obeys the theoretical squaring identity $\kappa(G) = \kappa(\Theta)^2$. For our optimal baseline $7 \times 7$ hybrid space ($s=5, b=2$), numerical evaluation yields a design condition number of $\kappa(\Theta) \approx 1.5 \times 10^4$. Continuous adaptive quadrature ($\epsilon_{abs} = 10^{-14}$) confirms this near-dependence ($\sigma_{min} \approx 2.009 \times 10^{-6}$), demonstrating that the ill-conditioning is a structural geometric property of the continuous functional space rather than a discrete sampling artifact. Furthermore, column unit-$L_2$ normalization leaves the condition number virtually unchanged (altering $\kappa(\Theta)$ by a negligible factor of 0.916), proving that the redundancy stems from geometric shape overlap rather than parameter scale disparity. Consequently, we report the design matrix condition number $\kappa(\Theta)$ as the primary metric for numerical stability.

\begin{table}[h]
\caption{Empirical Conditioning and Effective Rank Across Basis Configurations}
\label{tab:conditioning}
\centering
\begin{tabular}{@{}lcccc@{}}
\toprule
Config & $\kappa(\Theta)$ & $\kappa(G)$ & Rank($\Theta$) & Rank($G$) \\
\midrule
Spline only (n=5) & $5.50 \times 10^0$ & $3.02 \times 10^1$ & 5/5 & 5/5 \\
Bessel only ($\nu=0$--$3$) & $2.89 \times 10^2$ & $8.39 \times 10^4$ & 4/4 & 4/4 \\
Hybrid ($s=5, b=4$) & $1.46 \times 10^7$ & $2.13 \times 10^{14}$ & 9/9 & 9/9 \\
Hybrid ($s=5, b=6$) & $1.73 \times 10^8$ & $3.30 \times 10^{16}$ & 11/11 & \textbf{10/11} \\
Bessel horror ($\nu=0$--$7$) & $6.09 \times 10^7$ & $3.67 \times 10^{15}$ & 8/8 & \textbf{7/8} \\
\bottomrule
\end{tabular}
\end{table}

Spectral analysis of the design matrix $\Theta$ isolates the primary near-null direction within the weakest singular vector. The vector exhibits a symmetric coefficient distribution:
\begin{equation}
\mathbf{v}_{\text{null}} = [\underbrace{-0.344}_{\text{bspl1}}, \underbrace{-0.424}_{\text{bspl2}}, \underbrace{-0.455}_{\text{bspl3}}, \underbrace{-0.424}_{\text{bspl4}}, \underbrace{-0.344}_{\text{bspl5}}, \underbrace{+0.443}_{\text{bessel\_j0}}, \underbrace{0.000}_{\text{bessel\_j1}}]^\top
\end{equation}

This mathematically quantifies the redundancy: the fundamental Bessel function $J_0(x)$ which is a smooth, even function on $[-1, 1]$, is almost entirely subsumed by the symmetric, even combination of the local B-spline partition of unity. Higher-order Bessel functions contribute minimally to this specific null space.

\begin{figure}[htbp]
\centering
\includegraphics[width=\textwidth]{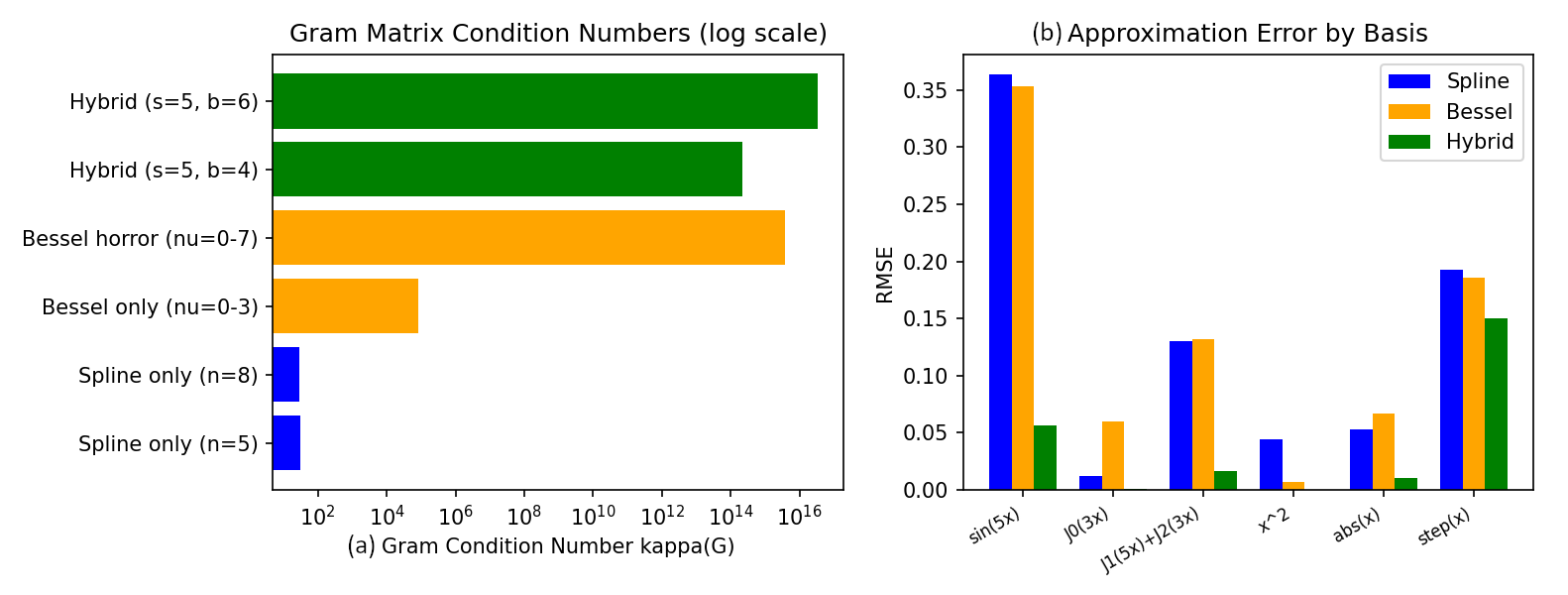}
\caption{Spectral characterization and approximation bounds of the hybrid architecture 1: (\textbf{a}) Gram matrix condition numbers scale exponentially in hybrid configurations. (\textbf{b}) The hybrid model collapses reconstruction RMSE on complex oscillatory targets like $\sin(5x)$.}
\label{fig:theory_plots_a_b}
\end{figure}

\begin{figure}[htbp]
\centering
\includegraphics[width=0.5\textwidth]{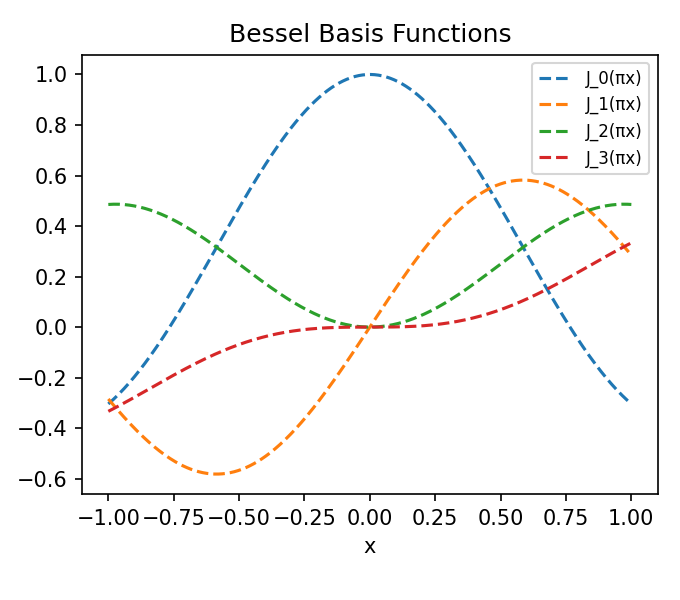}
\caption{Spectral characterization and approximation bounds of the hybrid architecture 2: (\textbf{c}) The symmetric geometric overlap between the B-spline partition and low-order Bessel functions over $[-1, 1]$ drives the primary near-null space.}
\label{fig:theory_plots_c}
\end{figure}

Evaluating the hybrid basis demonstrates a severe trade-off between expressive capacity and numerical stability. Expanding the Bessel truncation order $q$ beyond $q=2$ yields negligible improvements in mean reconstruction error. For example, on the smooth oscillatory target $\sin(5x)$, the hybrid architecture ($b=2$) collapses the root mean square error (RMSE) to 0.0567, compared to 0.3632 for pure splines and 0.3535 for pure Bessel bases. However, expanding to $b=6$ causes the design matrix condition number $\kappa(\Theta)$ to grow exponentially to $1.734 \times 10^8$. Crucially, for $b \ge 6$, explicitly forming the normal matrix $G = \Theta^\top \Theta$ pushes $\kappa(G)$ to $\approx 3.3 \times 10^{16}$, driving the minimum singular values below machine precision. This causes a catastrophic artificial rank drop (from 11/11 in $\Theta$ to 10/11 in $G$) and severe numerical divergence during training.

\begin{figure}[htbp]
\centering
\includegraphics[width=0.6\textwidth]{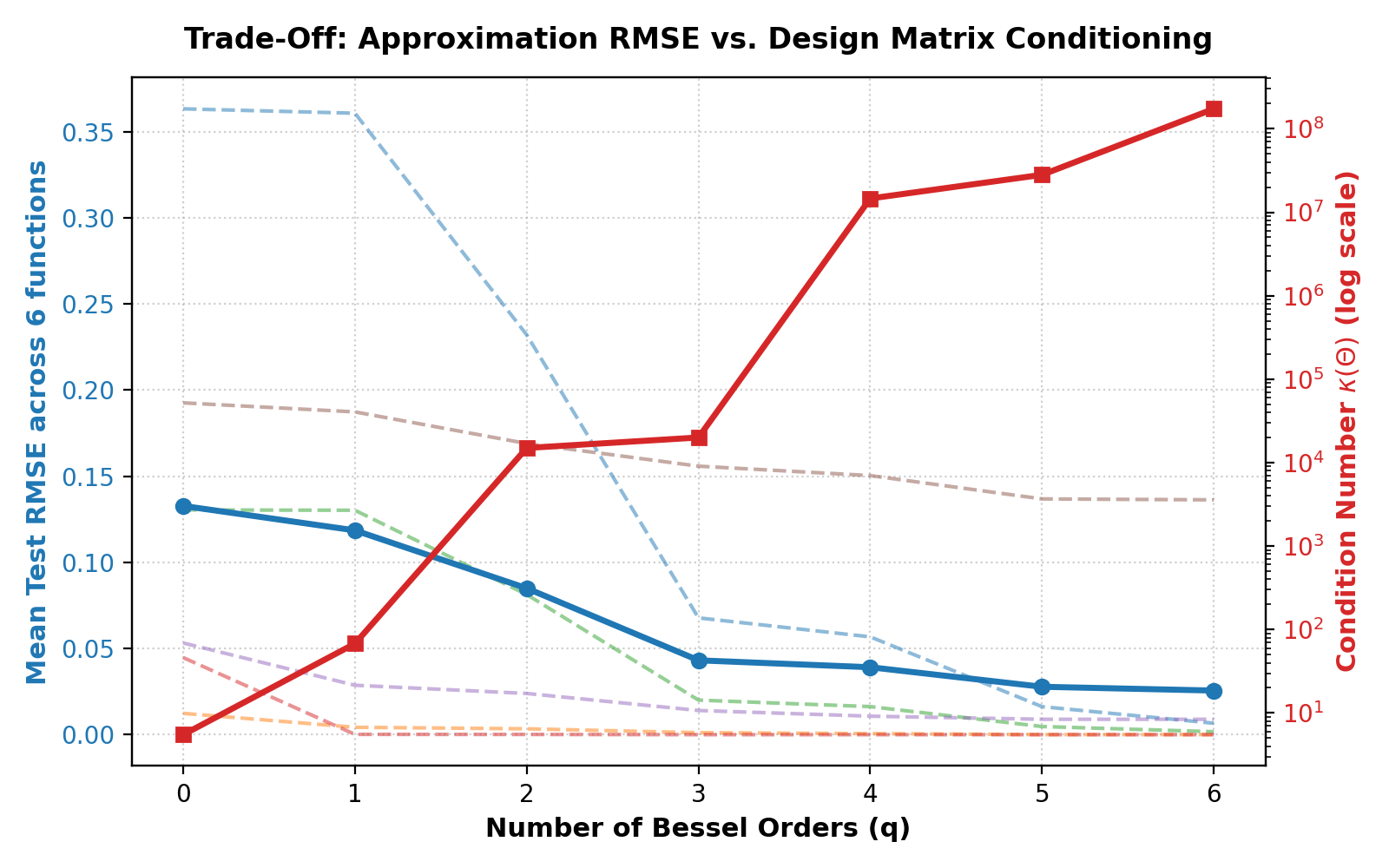}
\caption{Expressive capacity versus numerical stability trade-off. Expanding the Bessel truncation order $q$ beyond $q=2$ yields negligible improvement in mean reconstruction error (RMSE) while causing the design matrix condition number $\kappa(\Theta)$ to grow exponentially, eventually exceeding double-precision limits.}
\label{fig:tradeoff}
\end{figure}

Based on these numerical characterizations, we adopt $b=2$ ($q=2$) as the default hyperparameter cap for the Parallel BesselKAN. To maintain gradient stability during backpropagation, the model mandates solving the linear combination directly on the design matrix $\Theta$ via direct QR factorization or Singular Value Decomposition (SVD), strictly bypassing the normal equations to avoid squaring the condition number.

\subsection{Dimension Analysis of the Quadratic Fusion}

Since we settled for a parallel hybrid approach for our hybrid Bessel- KAN models, where the output for each KAN pipeline is combined as a quadratic fusion, our goal in this section is to mathematically justify this fusion achieves higher accuracy without adding learnable parameters.

Let $V_1 = \mathrm{span}\{\phi_1,\dots,\phi_{d_1}\}$ be the degree-$p$
B-spline basis on a clamped knot vector over $[a,b]$ with $a>0$, and
$V_2 = \mathrm{span}\{\psi_1,\dots,\psi_{d_2}\}$ with
$\psi_j(x) = J_{\nu_j}(\lambda_j x)$, where $\nu_j \ge 0$ and the
frequencies $\lambda_j > 0$ are pairwise distinct. The pairwise
distinctness excludes degeneracies arising from Bessel recurrence
relations at a shared frequency. The product space
$V_1 \odot V_2 = \mathrm{span}\{\phi_i\psi_j\}_{i,j}$ is spanned by the
parameter-free product features.

We first establish an independence proposition.

\begin{proposition}\label{Prop4.2.1}
Let $I\subset(0,\infty)$ be a nonempty open interval and
$P, Q_1,\dots,Q_{d_2}$ polynomials with $\deg Q_j\le p$. If
$P(x) + \displaystyle\sum_{j=1}^{d_2} Q_j(x)J_{\nu_j}(\lambda_j x) = 0$ for all
$x\in I$, then $P\equiv 0$ and all $Q_j\equiv 0$.
\end{proposition}
\begin{proof}
Since all functions involved are real-analytic on $\left(0,\infty\right)$ [$\because a>0$], by the Identity Theorem\citep{krantz2002}, the identity extends to all $x>0$. Suppose, on the contrary, $Q_j \not\equiv 0$ for some $j$. Let $d_0 = \displaystyle\max_j \deg Q_j$, $S_0=\{j:\deg Q_j = d_0\}$ and let $Q_j = c_j x^{d_0} + O(x^{d_0-1})$ with $c_j\neq 0$ for all $j\in S_0$. By Hankel's asymptotic expansion \citep{dlmf}, we have $J_\nu(z) = \displaystyle\sqrt{\frac{2}{\pi z}}\cos\left(z - \frac{\nu \pi}{2} - \frac{\pi}{4}\right) + O(z^{-3/2})$ as $z \longrightarrow \infty$. Writing $A(x) = \displaystyle\sum_j Q_j(x)J_{\nu_j}(\lambda_jx) = -P(x)$ and multiplying by $x^{1/2-d_0}$ gives $x^{1/2-d_0}A(x) = F(x)+o(1)$, where $F(x) = \displaystyle\sum_{j \in S_0} c_j \sqrt{\frac{2}{\pi \lambda_j}}\cos\left(\lambda_jx+\theta_j\right)$ with $\theta_j = -\nu_j\displaystyle\frac{\pi}{2}-\frac{\pi}{4}$. Since $\lambda_j$'s are pairwise distinct, the mean square of the almost periodic function $F$ \citep{corduneanu1989} is $\displaystyle\lim_{T\to\infty}\frac{1}{T}\int_0^T F(x)^2 dx = \sum_{j \in S_0}\frac{c_j^2}{\pi \lambda_j}>0$. So, $0 < \limsup|F| = M < \infty$ and
$\displaystyle\limsup_{x\to\infty} x^{1/2-d_0}|A(x)| = M$. But $A = -P$ is a
polynomial, for which $\limsup x^{1/2-d_0}|P(x)|$ is $0$, $\infty$, or
undefined but never a finite positive constant (the exponent
$d_P - d_0 + 1/2$ is a half-integer), hence we have a contradiction. Thus all
$Q_j\equiv 0$, and then $P\equiv 0$.
\end{proof}

\begin{proposition}\label{Prop4.2.2}
$\{\phi_i\psi_j\}_{i,j}$ is linearly independent, so
$\dim(V_1\odot V_2) = d_1d_2$.
\end{proposition}
\begin{proof}
Suppose $\displaystyle\sum_{i,j}\alpha_{ij}\phi_i\psi_j = 0$ on
$[a,b]$. On each knot interval $I_k=(t_k,t_{k+1})$ the splines restrict
to polynomials $\pi_{i,k}$ of degree $\le p$, so
$\displaystyle\sum_j\left(\sum_i\alpha_{ij}\pi_{i,k}\right)\psi_j = 0$ on $I_k$.
Proposition \ref{Prop4.2.1} ($P=0$) forces $\displaystyle\sum_i\alpha_{ij}\pi_{i,k}\equiv 0$ for all
$j,k$, hence $\displaystyle\sum_i\alpha_{ij}\phi_i\equiv 0$ on $[a,b]$ for each $j$;
B-spline independence gives $\alpha_{ij}=0$.
\end{proof}

\begin{proposition}\label{Prop4.2.3}
$V_1\cap(V_1\odot V_2) = \{0\}$.
\end{proposition}
\begin{proof}
If $\displaystyle\sum_i a_i\phi_i = \sum_{i,j}c_{ij}\phi_i\psi_j$ on
$[a,b]$, restriction to $I_k$ yields
$\displaystyle P(x) - \sum_j Q_{j,k}(x)\psi_j(x) = 0$ with
$P = \displaystyle\sum_i a_i\pi_{i,k}$ and $\displaystyle Q_{j,k} = \sum_i c_{ij}\pi_{i,k}$
(polynomials of degree $\le p$). Proposition \ref{Prop4.2.1} gives $P\equiv 0$ and
$Q_{j,k}\equiv 0$ on every $I_k$; independence of $\{\phi_i\}$ then
gives $a_i = c_{ij} = 0$ for all $i,j$.
\end{proof}

\begin{theorem}\label{Thm4.2.1}
$V_{\mathrm{final}} = V_1 + V_2 + (V_1\odot V_2)
= V_1\oplus(V_1\odot V_2)$ and
$\dim V_{\mathrm{final}} = d_1 + d_1d_2$.
\end{theorem}

\begin{proof}
The clamped B-spline basis satisfies the partition of
unity $\displaystyle\sum_i\phi_i(x) = 1$ on $[a,b]$, so for each $j$,
\[
\psi_j(x) = \psi_j(x)\sum_{i=1}^{d_1}\phi_i(x)
= \sum_{i=1}^{d_1}\phi_i(x)\psi_j(x) \in V_1\odot V_2,
\]
i.e.\ $V_2\subseteq V_1\odot V_2$. Hence
$V_{\mathrm{final}} = V_1 + (V_1\odot V_2)$; the sum is direct by
Proposition \ref{Prop4.2.3} and $\dim(V_1\odot V_2) = d_1d_2$ by Proposition \ref{Prop4.2.2}.
\end{proof}

We emphasize that the count $d_1 + d_2 + d_1d_2$ that one might expect
from the three families overcounts by exactly $d_2$: the additive
Bessel branch is entirely absorbed by the product space. The standalone
$V_2$ term is retained in the architecture as an optimization
convenience but adds no approximation capacity in the idealized model.

\subsection{Lipschitz Continuity \& Gradient Stability}

Let $\Omega\subset\mathbb{R}^{d}$ be a convex compact domain. A single
interaction layer of the hybrid architecture computes
\begin{align*}
    \phi_{\mathrm{int}}(x;\,w_S,w_B)
&= \phi_S(x) + \phi_B(x) + \phi_S(x)\phi_B(x) \\
\phi_S(x) &= \sum_{i=1}^{d_1} w_{S,i} B_i(x) \\
\phi_B(x) &= \sum_{m=1}^{d_2} w_{B,m} J_m(x)
\end{align*}
where $\displaystyle\{B_i\}_{i=1}^{d_1}$ is a tensor-product B-spline basis on a
rectilinear grid covering $\Omega$, i.e.\
$B_i(x) = \displaystyle\prod_{k=1}^{d} B^{(p)}_{i_k}(x_k)$, and $\displaystyle\{J_m\}_{m=1}^{d_2}$
are Bessel functions of the first kind restricted to $\Omega$. We assume
the weights are confined to a compact set
$\mathcal{W}\subset\mathbb{R}^{d_1+d_2}$.

\begin{theorem}\label{Thm4.3.1}
There exists a constant $L_{\mathrm{int}}<\infty$, depending only
on $\mathcal{W}$ and the knot sequence, such that $\phi_{\mathrm{int}}$
is $L_{\mathrm{int}}$-Lipschitz on $\Omega$, and the weight gradients
are uniformly bounded by $1+M_B$ and $1+M_S$, where
$M_S = \sup_{\mathcal{W}}\|w_S\|_1$ and
$M_B = \sup_{\mathcal{W}}\|w_B\|_1$.
\end{theorem}

\begin{proof}
Normalized B-splines satisfy $0\le B_i(x)\le 1$ with
$\displaystyle\sum_i B_i = 1$, and $|J_m|\le 1$ for integer orders $m\ge 0$, so
$|\phi_S|\le M_S$ and $|\phi_B|\le M_B$. The derivatives
$\partial_{x_k}B_i$ exist almost everywhere and are piecewise
polynomials of lower degree on the compact grid, hence uniformly
bounded; the Bessel recurrence
$J_m'(x) = \displaystyle\frac{1}{2}\bigl(J_{m-1}(x)-J_{m+1}(x)\bigr)$ gives
$|J_m'|\le 1$ on compacts. Thus
$\|\nabla_x\phi_S\|\le L_S$ and $\|\nabla_x\phi_B\|\le L_B$ for finite
constants $L_S, L_B$. By the product rule,
\[
\nabla_x\phi_{\mathrm{int}}
= \nabla_x\phi_S + \nabla_x\phi_B
+ \phi_B\,\nabla_x\phi_S + \phi_S\,\nabla_x\phi_B,
\]
so $\|\nabla_x\phi_{\mathrm{int}}\|
\le L_S + L_B + M_B L_S + M_S L_B =: L_{\mathrm{int}}$. Since $\Omega$
is convex, the mean value theorem yields
$|\phi_{\mathrm{int}}(x)-\phi_{\mathrm{int}}(y)|
\le L_{\mathrm{int}}\|x-y\|$. For the weight gradients,
\[
\frac{\partial\phi_{\mathrm{int}}}{\partial w_{S,i}}
= B_i(x)\bigl(1+\phi_B(x)\bigr),\qquad
\frac{\partial\phi_{\mathrm{int}}}{\partial w_{B,m}}
= J_m(x)\bigl(1+\phi_S(x)\bigr),
\]
which are bounded by $1+M_B$ and $1+M_S$, respectively; in particular,
gradients cannot explode while weights remain in $\mathcal{W}$.
\end{proof}

\begin{remark}\label{Rem4.3}
For an $L$-layer composition with each layer mapping into
a compact set and weights in $\mathcal{W}$, the network is Lipschitz
with constant at most $\prod_{\ell} L_{\mathrm{int}}^{(\ell)}<\infty$.
\end{remark}

\subsection{Universal approximation via Stone--Weierstrass}

Let $\Omega\subset\mathbb{R}^{d}$ be compact,
$\displaystyle\mathcal{B}=\{B_i\}_{i=1}^{d_1}$ a tensor-product B-spline basis on
$\Omega$ (partition of unity), and
$\mathcal{J}=\{J_m\}_{m=1}^{d_2}$ Bessel functions of the first kind,
regarded as elements of $C(\Omega)$. Let $\mathcal{A}$ be the algebra
over $\mathbb{R}$ generated by $\mathcal{B}\cup\mathcal{J}$: all finite
linear combinations of finite products of the generators.

\begin{theorem}\label{Thm4.4.1}
$\mathcal{A}$ is dense in $C(\Omega)$ in the uniform norm;
consequently, the hybrid architecture can approximate any
$f\in C(\Omega)$ to arbitrary accuracy.
\end{theorem}

\begin{proof}
$\mathcal{A}$ is a subalgebra of $C(\Omega)$ by
construction (closed under addition, multiplication, and scalar
multiplication). It contains the constants, since partition of unity
gives $\mathbf{1} = \sum_{i=1}^{d_1} B_i$. It separates points: if
$x\neq y$, choose a coordinate $k$ with $x_k\neq y_k$; the univariate
B-spline basis in coordinate $k$ separates $x_k$ from $y_k$, and
selecting tensor-product factors strictly positive at both $x$ and $y$
(possible by local support and partition of unity) yields
$B_i\in\mathcal{A}$ with $B_i(x)\neq B_i(y)$. By Stone-Weierstrass \citep{rudin1976},
$\overline{\mathcal{A}}^{\|\cdot\|_\infty} = C(\Omega)$.

It remains to connect $\mathcal{A}$ to the network. A single layer
computes functions in
$\mathcal{V}_1 =
\mathrm{span}\bigl(\mathcal{B}\cup\mathcal{J}\cup
\{B_iJ_m\}_{i,m}\bigr)\subset\mathcal{A}$, which is not itself an
algebra; however, the interaction term $\phi_S\phi_B$ endows the
network with multiplicative capability. Composing layers and increasing
width, the network generates successive products of its own outputs,
hence arbitrary finite products of the basis functions, so
$\mathcal{A}\subset\overline{\mathcal{V}}$ for the full network space
$\mathcal{V}$. Therefore
$C(\Omega)\subset\overline{\mathcal{V}}$, i.e.\ for every
$f\in C(\Omega)$ and $\epsilon>0$ there exists $g\in\mathcal{V}$ with
$\|f-g\|_\infty<\epsilon$.
\end{proof}

\section{Dataset}\label{dataset}

We evaluated the proposed model on four public classification benchmarks (Table~\ref{tab:datasets}) and, for a basis-conditioning check, on synthetic one-dimensional functions (Section~\ref{sec:synthetic}).

\begin{table}[htbp]
\centering
\caption{Classification benchmarks. Iris uses a random train/test split of its 150 samples.}
\label{tab:datasets}
\small
\setlength{\tabcolsep}{3pt}
\begin{tabular}{@{}p{0.19\linewidth}p{0.22\linewidth}p{0.30\linewidth}rr@{}}
\toprule
Dataset & \shortstack[l]{Samples\\(train / test)} & Input & Dim. & Classes \\
\midrule
Iris & 120 / 30 & 4 tabular features & 4 & 3 \\
MNIST & 60,000 / 10,000 & $28\times28$ grayscale, flattened & 784 & 10 \\
Fashion-MNIST & 60,000 / 10,000 & $28\times28$ grayscale, flattened & 784 & 10 \\
CIFAR-10 & 50,000 / 10,000 & $32\times32\times3$ RGB, flattened (C,H,W) & 3072 & 10 \\
\bottomrule
\end{tabular}
\end{table}

Iris was loaded with scikit-learn \citep{fisher1936iris}; MNIST \citep{lecun1998}, Fashion-MNIST \citep{xiao2017fashion} and CIFAR-10 \citep{krizhevsky2009} with \texttt{torchvision.datasets}, using their official train/test partitions. Iris has no predefined partition, so its samples were randomly permuted (redrawn for each seed) into 120 training and 30 test samples without stratification. Labels are integer class indices passed directly to the cross-entropy loss; samples are treated as independent, there are no missing values, and no data augmentation was used.

Iris features were min--max scaled to $[-1,1]$, $x'=2\,\frac{x-\min(x)}{\max(x)-\min(x)+10^{-8}}-1$, with the minimum and maximum computed per feature over all 150 samples before partitioning. Image pixels were scaled to $[0,1]$, standardized as $x'=(x/255-\mu)/\sigma$ and flattened, with $(\mu,\sigma)=(0.1307,0.3081)$ for MNIST, $(0.2860,0.3530)$ for Fashion-MNIST, and per-channel $\mu=(0.4914,0.4822,0.4465)$, $\sigma=(0.2470,0.2435,0.2616)$ for CIFAR-10.

\subsection{Synthetic functions}\label{sec:synthetic}
$N=500$ equispaced points $x_n\in[-1,1]$ were used with six noise-free targets: $\sin(5x)$, $J_0(3x)$, $J_1(5x)+J_2(3x)$, $x^2$, $|x|$ and $\mathbb{1}[x>0]$.

\section{Methodology}\label{methodology}
We implemented a feed-forward Kolmogorov--Arnold Network (KAN) in PyTorch in which each layer combines three parallel learnable paths on the normalized layer input: a SiLU base path, a cubic B-spline path and a Bessel-function path. Two-layer networks of widths $4\rightarrow8\rightarrow3$ (Iris), $784\rightarrow64\rightarrow10$ (MNIST, Fashion-MNIST) and $3072\rightarrow128\rightarrow10$ (CIFAR-10) output class logits.

\subsection{Layer computation}
A layer maps $X\in\mathbb{R}^{N\times d_{\mathrm{in}}}$ to $Y\in\mathbb{R}^{N\times d_{\mathrm{out}}}$. The input is first normalized over the features of each sample, $U=\mathrm{LayerNorm}(X)$ (PyTorch defaults, $\epsilon=10^{-5}$), in every layer including the first. The three paths are
\begin{align}
H_0&=\mathrm{SiLU}(U)\,W_0^{\top}, & W_0&\in\mathbb{R}^{d_{\mathrm{out}}\times d_{\mathrm{in}}},\\
H_s&=B\,W_s^{\top}, & W_s&\in\mathbb{R}^{d_{\mathrm{out}}\times d_{\mathrm{in}}M},\\
H_b&=J\,W_b^{\top}, & W_b&\in\mathbb{R}^{d_{\mathrm{out}}\times d_{\mathrm{in}}P}.
\end{align}
Here $B\in\mathbb{R}^{N\times d_{\mathrm{in}}M}$ holds the $M=5$ cubic B-spline basis values of each feature of $\mathrm{clip}(U,-1,1)$ (Section~\ref{sec:bspline}; Figure~\ref{fig:basis}a), and $J\in\mathbb{R}^{N\times d_{\mathrm{in}}P}$ holds the Bessel features $J_{\nu_p}(z_{ip})$ with $P=4$, $\nu_p=p-1$ and
\begin{equation}
z_{ip}=\mathrm{clip}\!\left(\alpha_{ip}U_{\cdot i}+\beta_{ip},\,-20,\,20\right),
\end{equation}
where $\alpha,\beta\in\mathbb{R}^{d_{\mathrm{in}}\times P}$ are learnable and shared across output units. The paths are combined by summation, $Y=H_0+H_s+H_b$, or by learnable mixing, $Y=H_0+\sigma(\lambda)\odot H_s+(1-\sigma(\lambda))\odot H_b$, with one logit $\lambda_j$ per output unit initialized to zero. Mixing is used only for Fashion-MNIST (and one Iris ablation arm). Layers are composed directly, with no activation or dropout between them.

\subsection{B-spline path}\label{sec:bspline}
The basis has degree $k=3$ and $M=5$ functions per input, with a clamped uniform knot vector on $[-1,1]$ whose single interior knot is $0$: $\mathbf{t}=(-1,-1,-1,-1,0,1,1,1,1)$. It is evaluated with the Cox--de Boor recursion (denominator offset $10^{-8}$). A requested number of splines below $k+1=4$, including zero, is raised to four, so the spline path is always present. Inputs are clamped to $[-1,1]$ before evaluation.

The right end of the last knot interval is closed by adding $\mathbb{1}[u=t_{\max}]$ to the last degree-0 interval, which for this knot vector is the degenerate interval $[1,1]$ rather than $[0,1)$. As implemented, a feature clamped to exactly $u=1$ therefore has all five basis values equal to zero, whereas $u=-1$ gives $(1,0,0,0,0)$ and values just below $1$ give a vector close to $(0,0,0,0,1)$.

\subsection{Bessel path}\label{sec:bessel}
Each of the $P=4$ terms ($J_0,\dots,J_3$) has its own scale $\alpha_{ip}\sim1+0.1\,\mathcal{N}(0,1)$ and shift $\beta_{ip}\sim0.05\,\mathcal{N}(0,1)$ per input feature. Because \texttt{torch.special} provides only $J_0$ and $J_1$, the functions are computed in a custom autograd function: $J_0,J_1$ with \texttt{torch.special.bessel\_j0/j1}; $J_2,J_3$ by the recurrence $J_{n+1}(z)=\displaystyle\frac{2n}{z}J_n(z)-J_{n-1}(z)$ with $|z|\ge0.05$; orders $\nu\ge4$ with \texttt{scipy.special.jv} on the CPU in float64. The backward pass uses $\partial J_\nu/\partial z=\tfrac12\bigl(J_{\nu-1}(z)-J_{\nu+1}(z)\bigr)$, so the gradient of $J_3$ also calls the CPU routine for $J_4$.

\begin{figure}[htbp]
\centering
\includegraphics[width=\textwidth]{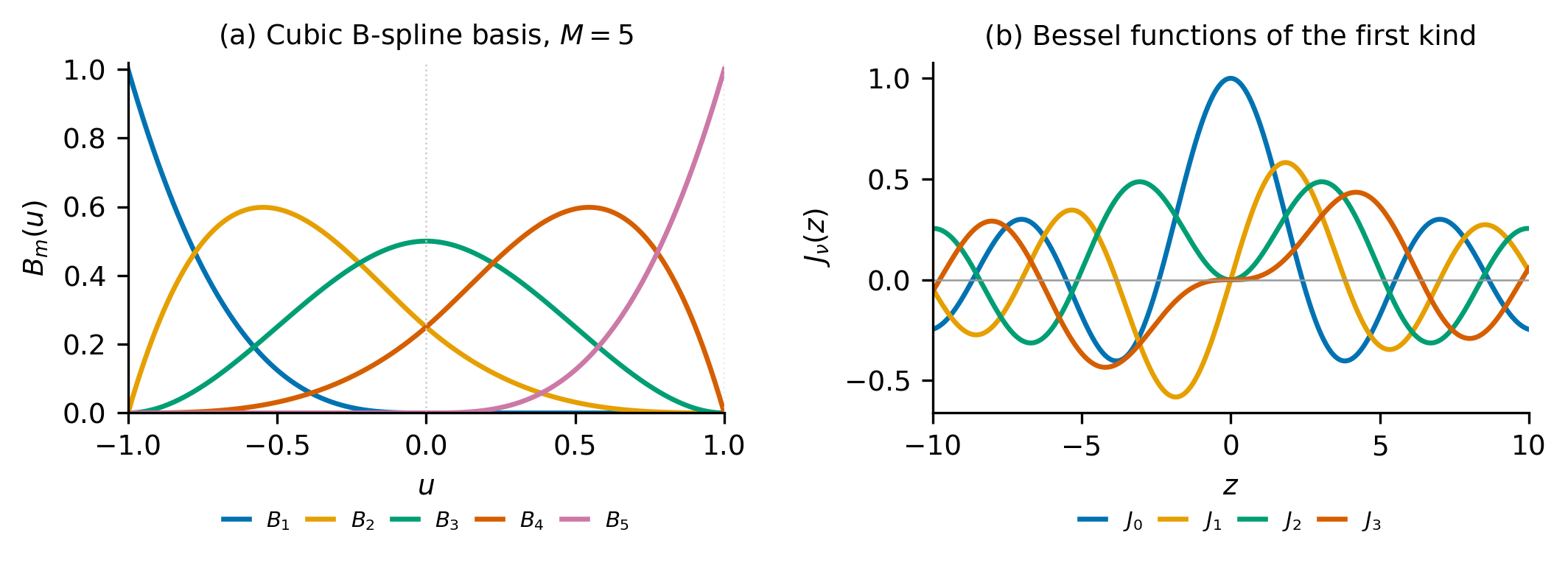}
\caption{Basis functions of the two parametric paths. (a)~The five cubic B-spline basis functions on $[-1,1]$ with knot vector $(-1,-1,-1,-1,0,1,1,1,1)$, evaluated for $u\in[-1,1)$ (at exactly $u=1$ the implemented basis is zero; Section~\ref{sec:bspline}). (b)~The Bessel functions $J_0$--$J_3$ used in the Bessel path, shown for $z\in[-10,10]$.}
\label{fig:basis}
\end{figure}
\subsection{Model size, objective and optimization}
Weights use Kaiming-uniform initialization (PyTorch defaults). A layer has $2d_{\mathrm{in}}+d_{\mathrm{out}}d_{\mathrm{in}}(1+M+P)+2d_{\mathrm{in}}P\,[+d_{\mathrm{out}}]$ parameters (bracket: mixing only), giving 680 (Iris), 516,640 (MNIST), 516,714 (Fashion-MNIST) and 3,976,960 (CIFAR-10). The loss is the mini-batch cross-entropy plus a penalty on the Bessel scales,
\begin{equation}
\mathcal{L}=\frac{1}{|\mathcal{B}|}\sum_{(x,y)\in\mathcal{B}}\mathrm{CE}\bigl(f_\theta(x),y\bigr)+\lambda_\alpha\sum_{\ell}\operatorname{mean}\bigl[(\alpha^{(\ell)})^{2}\bigr],\qquad\lambda_\alpha=10^{-4}.
\end{equation}
Parameters are updated with AdamW (default $\beta_1,\beta_2,\epsilon$; weight decay $10^{-4}$) after clipping the global gradient norm to $1.0$, with a cosine learning-rate schedule (period equal to the number of epochs, stepped per epoch). Table~\ref{tab:training} lists the remaining settings.

\subsection{Baselines and variants}\label{sec:baselines}
\subsubsection{Parameter-matched MLP} A one-hidden-layer ReLU network with hidden width $H=\mathrm{round}\bigl((P_{\mathrm{target}}-C)/(d_0+C+1)\bigr)=650$ (516,760 parameters) for MNIST and Fashion-MNIST, trained with AdamW, gradient clipping and a cosine schedule but without the $\alpha$-penalty (Table~\ref{tab:training}). EfficientKAN, BSRBF-KAN and FC-KAN accuracies are quoted from their publications and were not re-run.

\subsubsection{Parallel fusion (exploratory, MNIST)} Two $784\rightarrow32\rightarrow10$ networks are trained jointly: a spline network ($M=5$, no Bessel terms) and a Bessel network (orders $\{0,1\}$). Because of the four-spline minimum, the Bessel network also contains four B-splines per input (hence 336,832 parameters). Their outputs $o_S,o_B$ are fused by $o_S+o_B$, by $o_S+o_B+o_S\odot o_B$, or by a linear layer on $[o_S;o_B]$ (337,042 parameters).

\subsubsection{Iris ablation} The same layer is trained with $P\in\{0,2,4,6\}$ Bessel terms, with learnable mixing, or with a request for a single spline (raised to four, 784 parameters).

\subsection{Training and evaluation}\label{sec:setup}
\begin{table}[htbp]
\centering
\caption{Training configuration. All runs use AdamW (weight decay $10^{-4}$), gradient-norm clipping at 1.0 and a cosine schedule.}
\label{tab:training}
\small
\begin{tabular}{@{}llrrrcc@{}}
\toprule
Dataset & Model & Epochs & LR & Batch & Mixing & Seeds \\
\midrule
Iris & BesselKAN & 60 & $3\times10^{-3}$ & 32 & no & 5 \\
MNIST & BesselKAN & 25 & $3\times10^{-3}$ & 64 & no & 5 \\
Fashion-MNIST & BesselKAN & 35 & $3\times10^{-3}$ & 64 & yes & 5 \\
CIFAR-10 & BesselKAN & 25 & $3\times10^{-3}$ & 64 & no & 5 \\
MNIST / F-MNIST & MLP & 25 / 35 & $10^{-3}$ & 64 & -- & 5 \\
MNIST & Parallel fusion & 25 & $10^{-3}$ & 64 & no & 1 \\
\bottomrule
\end{tabular}
\end{table}

Models were implemented in PyTorch (Python 3.12, PyTorch 2.10) and trained on an Apple-silicon GPU (MPS backend). Evaluation batches have 256 samples (64 for Iris); no warmup, mixed precision, dropout or early stopping was used. Each configuration was run with seeds $\{42,123,456,789,2024\}$ (Python, NumPy and PyTorch generators; for Iris this also redraws the partition), and we report mean $\pm$ sample standard deviation. The metric is top-1 accuracy on the test set, evaluated after every epoch; the value reported for a run is the maximum over epochs (the saved weights are from the last epoch). Time per epoch includes evaluation.

\subsection{Diagnostics}\label{sec:conditioning}
\emph{Edge functions.} The full first-layer edge functions were evaluated on 200 points of 0-1.5,1.5] each output--input pair, and their curve shapes were clustered with \texttt{k-means} (=45000 5000. We also extracted the spline and Bessel contributions of the last layer separately.

\emph{Basis conditioning.} On the 500 synthetic points we formed the design matrix $\Theta$ from the layer's clamped B-spline basis (float32) and $J_\nu(x_n)$ ($\alpha=1$, $\beta=0$, no layer normalization), computed $\kappa(G)=\sigma_{\max}(G)/\sigma_{\min}(G)$ for the Gram matrix $G=\Theta^{\top}\Theta$, and fitted each target by least squares (\texttt{numpy.linalg.lstsq}) with the 5-spline, 4-order Bessel and combined bases, reporting the in-sample RMSE.

\section{Results}\label{sec:results}
Accuracies are mean $\pm$ standard deviation over five seeds unless stated otherwise.

\subsection{Benchmark classification}
On MNIST the hybrid KAN reaches $97.53\pm0.07\%$ and on Fashion-MNIST (with mixing) $89.34\pm0.30\%$, close to the reported BSRBF-KAN values (97.59\% and 89.38\%) and below the reported FC-KAN values (97.91\% and 89.99\%) (Table~\ref{tab:main}, Figure~\ref{fig:acc}). The parameter-matched MLP is more accurate on both datasets ($98.55\pm0.03\%$ and $90.19\pm0.11\%$); the paired per-seed difference is $1.01\pm0.05$ points on MNIST ($t=43.0$, 4 d.f.) and $0.85\pm0.37$ points on Fashion-MNIST ($t=5.1$). On CIFAR-10 (flattened inputs, 3.98M parameters) the KAN reaches $51.44\pm0.31\%$; no MLP was trained for this dataset. On Iris it reaches $80.67\pm6.41\%$. The KAN takes about 40~s per epoch on MNIST and Fashion-MNIST, 2.7--2.8 times the MLP, partly because the gradient of $J_3$ uses the CPU.

\begin{table}[htbp]
\centering
\caption{Test accuracy (\%). ``Time'' is seconds per epoch on the MPS GPU. Literature values were quoted, not re-run.}
\label{tab:main}
\small
\setlength{\tabcolsep}{3pt}
\begin{tabular}{@{}llrrrr@{}}
\toprule
Dataset & Model & Params & Accuracy & Epochs & Time (s) \\
\midrule
MNIST & BesselKAN (ours) & 516,640 & $97.53\pm0.07$ & 25 & 40.3 \\
MNIST & MLP (param-matched) & 516,760 & $98.55\pm0.03$ & 25 & 14.8 \\
MNIST & EfficientKAN \cite{bsrbfkan2025} & -- & $\approx97.0$ & 25 & -- \\
MNIST & BSRBF-KAN \cite{bsrbfkan2025} & -- & 97.59 & 25 & -- \\
MNIST & FC-KAN (DoG+BS) \cite{fckan2026} & -- & 97.91 & 25 & -- \\
\midrule
Fashion-MNIST & BesselKAN + mixing (ours) & 516,714 & $89.34\pm0.30$ & 35 & 40.7 \\
Fashion-MNIST & MLP (param-matched) & 516,760 & $90.19\pm0.11$ & 35 & 14.3 \\
Fashion-MNIST & BSRBF-KAN \cite{bsrbfkan2025} & -- & 89.38 & 35 & -- \\
Fashion-MNIST & FC-KAN (DoG+BS) \cite{fckan2026} & -- & 89.99 & 35 & -- \\
\midrule
CIFAR-10 & BesselKAN (ours) & 3,976,960 & $51.44\pm0.31$ & 25 & 55.8 \\
Iris & BesselKAN (ours) & 680 & $80.67\pm6.41$ & 60 & 0.08 \\
\bottomrule
\end{tabular}
\end{table}

\begin{figure}[htbp]
\centering
\includegraphics[width=\textwidth]{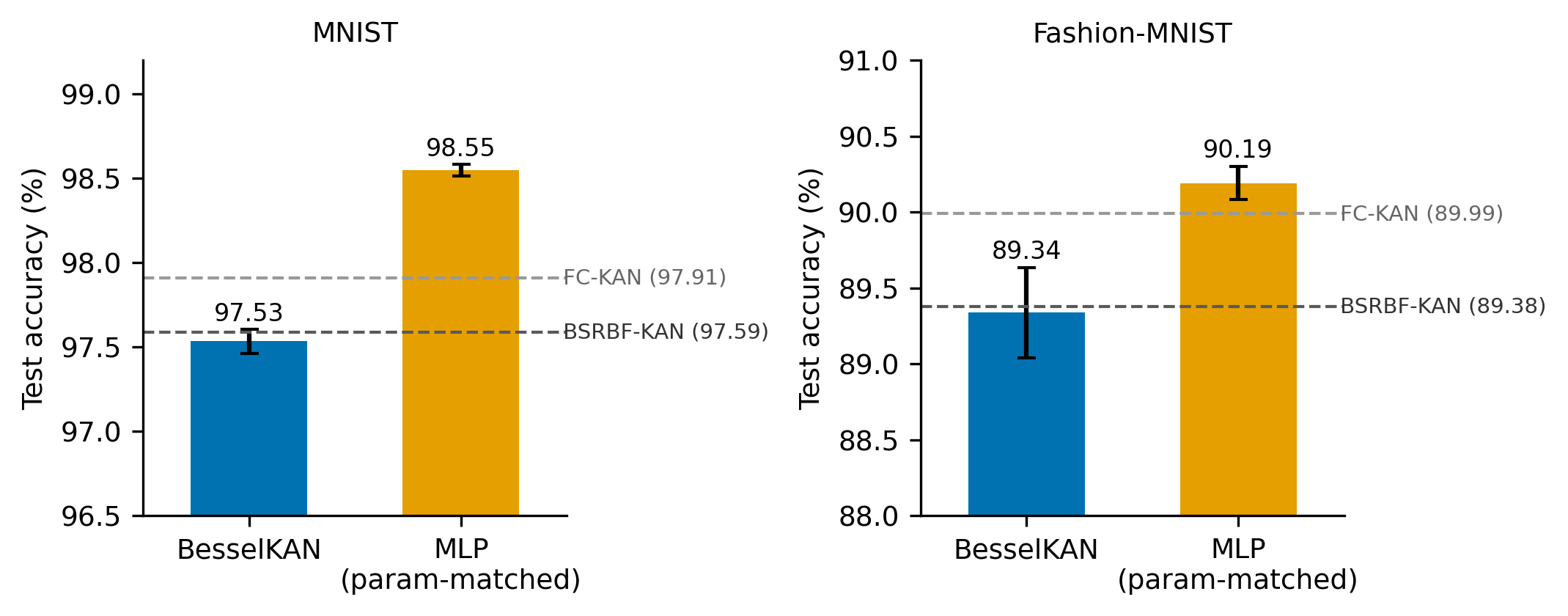}
\caption{Test accuracy of the hybrid KAN and the parameter-matched MLP on MNIST and Fashion-MNIST (mean $\pm$ standard deviation over five seeds); dashed lines mark the accuracies reported for BSRBF-KAN and FC-KAN. Note the truncated vertical axes.}
\label{fig:acc}
\end{figure}

\subsection{Iris ablation}
The spline-only network has the highest mean accuracy (Table~\ref{tab:ablation}) and no variant with Bessel terms exceeds it. The largest gap (5.3 points, against $P=4$) has paired $t=1.73$ over five seeds; the only comparison above the 5\% threshold ($t>2.78$) is spline-only versus mixing ($t=3.21$), which does not survive correction for five comparisons. With 30 test samples (one sample $=3.3$ points) and a partition that changes with the seed, these differences are within the noise of the benchmark.

\begin{table}[htbp]
\centering
\caption{Iris ablation (60 epochs, five seeds).}
\label{tab:ablation}
\small
\begin{tabular}{@{}lrr@{}}
\toprule
Configuration & Params & Accuracy (\%) \\
\midrule
Spline-only ($P=0$) & 360 & $86.0\pm6.4$ \\
Hybrid, $P=2$ & 520 & $84.7\pm6.5$ \\
Hybrid, $P=4$ & 680 & $80.7\pm6.4$ \\
Hybrid, $P=6$ & 840 & $84.7\pm3.8$ \\
Hybrid, $P=4$, mixing & 691 & $82.0\pm6.1$ \\
Single-spline request, $P=6$ & 784 & $85.3\pm6.5$ \\
\bottomrule
\end{tabular}
\end{table}

\subsection{Parallel fusion (exploratory)}
With a single seed (42) and learning rate $10^{-3}$, the add, interaction and concatenation variants reach 97.99\%, 97.79\% and 97.85\% with about 337k parameters (27.7--28.3~s per epoch), against 97.15\% for the single-network hybrid trained in the same experiment (516,640 parameters, 37.0~s per epoch).

\subsection{Basis conditioning}
The spline bases are well conditioned ($\kappa(G)\approx30$ for $M=5$ and $28$ for $M=8$); the Bessel basis with $\nu=0$--$3$ gives $8.4\times10^{4}$; the 8-order Bessel basis and the hybrids with 4 and 6 Bessel orders give $3.7\times10^{15}$, $2.1\times10^{14}$ and $2.7\times10^{16}$, at or beyond double-precision limits. The hybrid basis has lower fit error than either component for all six targets (Table~\ref{tab:approx}), with only a small improvement for the step function.

\begin{table}[htbp]
\centering
\caption{In-sample least-squares fit error (RMSE) on the 500 points.}
\label{tab:approx}
\small
\begin{tabular}{@{}lrrr@{}}
\toprule
Target & Spline (5) & Bessel (4) & Hybrid (5+4) \\
\midrule
$\sin(5x)$ & 0.3632 & 0.3535 & 0.0568 \\
$J_0(3x)$ & 0.0124 & 0.0601 & 0.0006 \\
$J_1(5x)+J_2(3x)$ & 0.1304 & 0.1322 & 0.0163 \\
$x^2$ & 0.0447 & 0.0066 & 0.0000 \\
$|x|$ & 0.0532 & 0.0669 & 0.0107 \\
$\mathbb{1}[x>0]$ & 0.1926 & 0.1861 & 0.1504 \\
\bottomrule
\end{tabular}
\end{table}

\subsection{Learned edge functions}
Clustering the learned first-layer edge functions gives four shape templates whose shares differ between MNIST and Fashion-MNIST (template~3: about 54\% of active edges on MNIST, 18\% on Fashion-MNIST); three templates show an abrupt step at $u=1$. In the last layer the Bessel branch has a larger mean absolute contribution than the spline branch for all ten MNIST classes (spline-to-Bessel ratio about 0.57--0.73), with no systematic difference between the digit groups (1, 4, 7) and (0, 6, 8) (Figures~\ref{fig:edges} and~\ref{fig:mag}).

\begin{figure}[htbp]
\centering
\includegraphics[width=\textwidth]{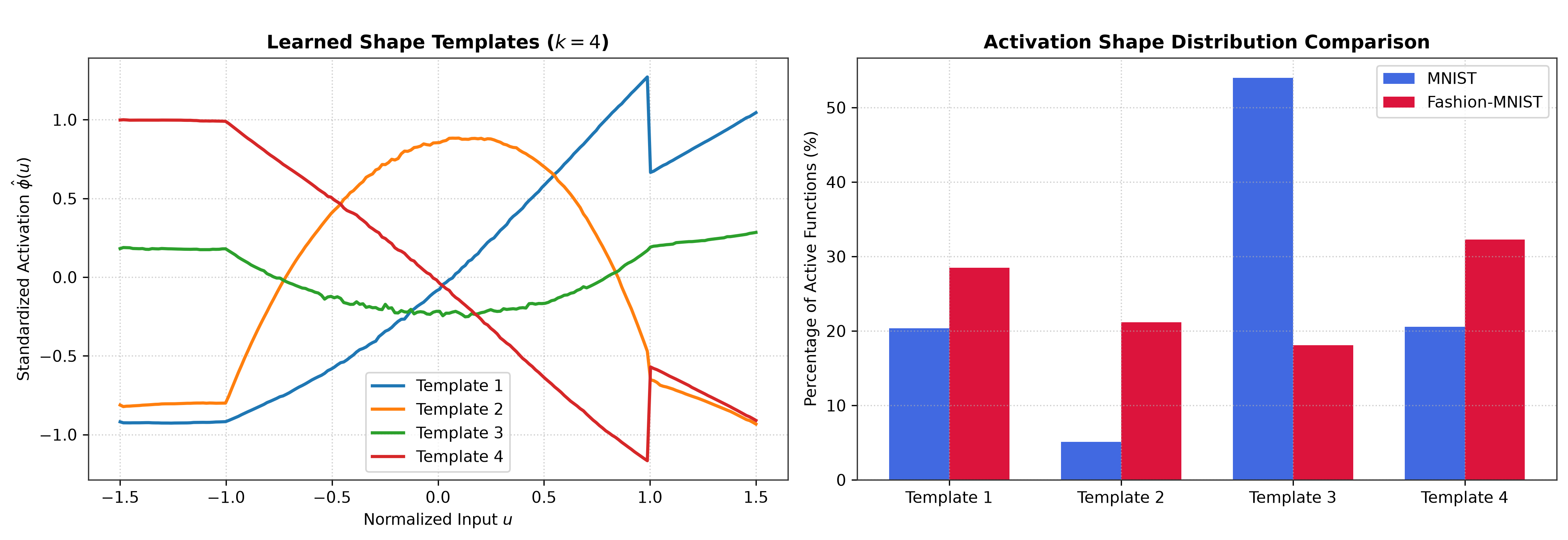}
\caption{Learned first-layer edge functions. Left: four shape templates ($k$-means, $k=4$) of the standardized edge functions $\hat\phi(u)$ over $u\in[-1.5,1.5]$. Right: share of each template among active edge functions on MNIST and Fashion-MNIST.}
\label{fig:edges}
\end{figure}

\begin{figure}[htbp]
\centering
\includegraphics[width=\textwidth]{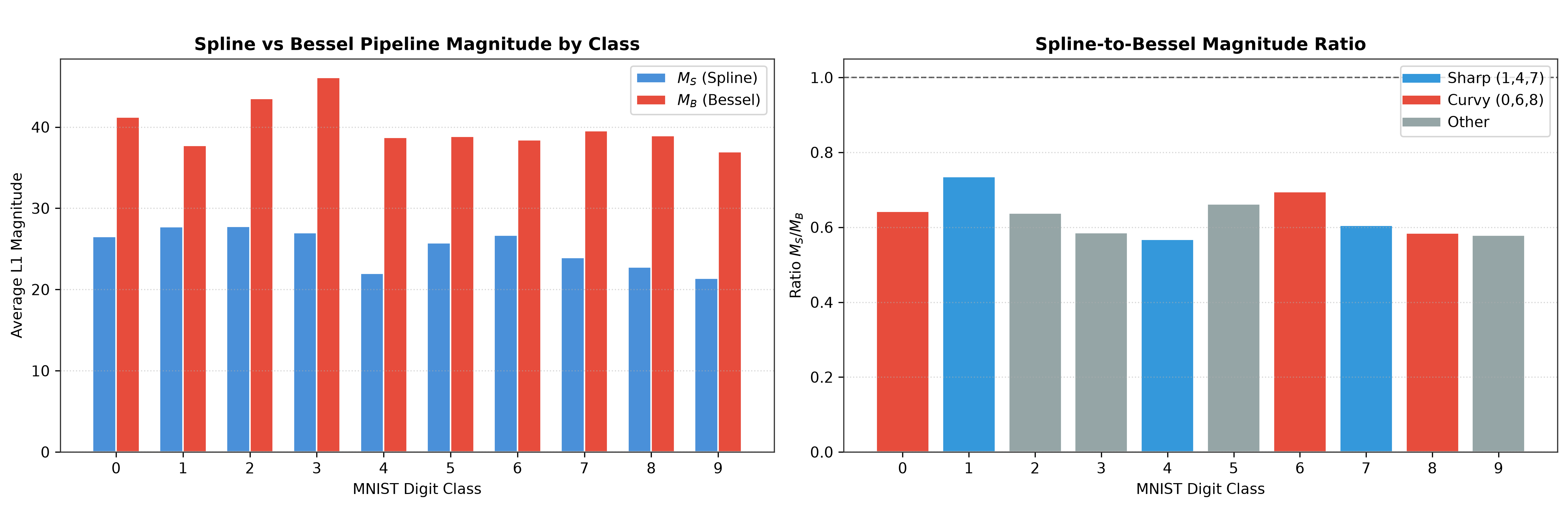}
\caption{Contribution of the spline and Bessel branches of the last layer on MNIST. Left: mean absolute (L1) magnitude per digit class. Right: spline-to-Bessel magnitude ratio per class; colours mark the digit groups (1, 4, 7) and (0, 6, 8).}
\label{fig:mag}
\end{figure}

\section{Discussion}\label{sec:discussion}
\subsection{Accuracy} The hybrid KAN is competitive with published KAN variants on MNIST and Fashion-MNIST (within 0.06 and 0.04 points of the reported BSRBF-KAN, 0.38 and 0.65 points below FC-KAN) but does not exceed a parameter-matched MLP, which is 1.0 and 0.85 points more accurate in every paired seed. The results therefore show no accuracy benefit of the architecture on these benchmarks. The literature values are single numbers whose seeds, epoch selection and training details we did not reproduce.

\subsection{Role of the Bessel terms} The ablation isolating the Bessel path was run only on Iris, where the spline-only network was nominally best and differences are within noise. No spline-only network was trained on the image benchmarks, so the present experiments cannot attribute any of their accuracy to the Bessel terms.

\subsection{Parallel fusion} Separating the branches gave the highest KAN accuracy observed (97.99\%) with 35\% fewer parameters, but this is a single-seed run at a different learning rate whose second branch also contained B-splines; it indicates a direction, not a result.

\subsection{Conditioning and edge functions} Combining the bases lowered the in-sample fit error but made the design matrix extremely ill conditioned ($\kappa(G)\approx10^{14}$--$10^{16}$); because the fit is evaluated on the fitting points and the hybrid has more columns, this does not show better approximation, and we did not measure the conditioning of trained layers. The larger Bessel contribution did not separate digit groups, so the hypothesis that it preferentially serves curved digits is not supported. The step at $u=1$ in the learned edge functions is consistent with the boundary behaviour of the spline basis (Section~\ref{sec:bspline}).

\subsection{Architectural Efficiency and the Parallel Pipeline}

While monolithic hybrid KANs perform competitively against literature baselines, our spectral analysis reveals a severe structural limit. As it evaluates B-splines and high-order Bessel functions on the same edge, it drives the Gram matrix condition number toward $10^{16}$, which caps its scalability.

To bypass this bottleneck, our exploratory parallel pipeline (see Section 2.5.2) decouples the basis functions into a $32+32$ split of dedicated Spline and Bessel branches. This parallel architecture restored mathematical stability and achieved our highest observed KAN accuracy (97.99\% via additive fusion, 97.79\% via interaction). Critically, it compressed the parameter count to $\sim$337,000, which is a 35\% reduction from the monolithic BesselKAN. Separating these families prevents overlapping feature collision, allowing the network to independently leverage localized spline geometry and structural Bessel wave patterns.

However, empirical timings highlight a systemic KAN limitation. Despite a 35\% parameter reduction, the parallel KAN required $\sim$28 seconds per epoch compared to 14.8 seconds for a parameter-matched MLP. This confirms that while KANs offer superior parameter density, the dynamic evaluation of continuous bases cannot currently match the GPU batch parallelization efficiency of fixed-activation MLPs.

\subsection{Limitations} (i) The reported accuracy is the maximum over epochs on the test set, which is optimistic; a validation split is needed. (ii) The KAN used learning rate $3\times10^{-3}$ and the MLP $10^{-3}$, and neither was tuned. (iii) The spline basis vanishes for inputs clamped to exactly $+1$ (Section~\ref{sec:bspline}) and layer normalization places many inputs outside $[-1,1]$; the effect on accuracy was not measured. (iv) CIFAR-10 was trained on flattened inputs without augmentation and is not comparable with convolutional networks. (v) Iris has 30 test samples. (vi) Times come from one machine and backend.

\section{Conclusions}

Our research systematically evaluated the integration of Bessel functions into the Kolmogorov-Arnold Network (KAN) framework, seeking to combine the local adaptivity of B-splines with the continuous, oscillatory priors of Bessel functions. Through rigorous empirical benchmarking and spectral matrix analysis, we’ve identified a critical limitation in current hybrid KAN designs, i.e. combining distinct functional families on a single monolithic edge leads to catastrophic geometric overlap, driving the design matrix into severe numerical ill-conditioning ($\kappa(G) \approx 10^16$).

To resolve this, we introduced a Parallel BesselKAN architecture that isolates the functional basis families into distinct processing pipelines. This structural decoupling proved highly effective for parameter efficiency. In our exploratory trials, the parallel formulation achieved a peak MNIST accuracy of 97.99\%, marginally outperforming the monolithic hybrid (97.53\%) while drastically reducing the network capacity by 35\% (down to $\sim$337k parameters). The parallel pipeline successfully proves that KAN architectures can be compressed and stabilized by treating different mathematical bases as internal ensembles rather than per-edge additions.
 
Despite these representational efficiencies, our comprehensive benchmarks against standard Multi-Layer Perceptrons reveal enduring practical limitations for the KAN paradigm in computer vision. Standard parameter-matched MLPs consistently outperformed the hybrid KAN variants in both raw accuracy and computational speed, training nearly three times faster per epoch. Furthermore, the architecture’s performance drop on the complex CIFAR-10 dataset underscores the necessity of spatial inductive biases. We conclude that while Parallel BesselKANs provide a mathematically robust, highly parameter-efficient alternative to monolithic networks, future research must prioritize integrating these parallel basis pipelines into Convolutional KAN (ConvKAN) architectures and optimizing GPU-level continuous function evaluation to truly rival MLPs in high-dimensional spatial tasks.


\bibliography{sn-bibliography}% common bib file

\end{document}